\documentclass[numsec,webpdf,modern,medium,namedate]{oup-authoring-template}

\onecolumn

\graphicspath{{Fig/}}

\theoremstyle{thmstyleone}%
\newtheorem{theorem}{Theorem}
\newtheorem{proposition}[theorem]{Proposition}%
\theoremstyle{thmstyletwo}%
\theoremstyle{thmstylethree}%

\usepackage[doublespacing]{setspace}
\usepackage[fontsize=12pt]{fontsize}

\begin{document}

\journaltitle{Private Manuscript}
\copyrightyear{2025}
\pubyear{XXXX}
\appnotes{Original article}

\firstpage{1}


\title[AI Security \& Alignment Limitations]{Robust AI Security and Alignment: A Sisyphean Endeavor?}

\author[1,$\ast$]{Apostol Vassilev\ORCID{0000-0002-9081-3042}}

\authormark{Apostol Vassilev}

\address[1]{\orgdiv{CSD/ITL}, \orgname{NIST}, \orgaddress{\street{100 Bureau Drive, Gaithersburg}, \postcode{20899}, \state{MD}, \country{USA}}}

\corresp[$\ast$]{Address for correspondence. Apostol Vassilev, NIST, Gaithersburg, MD 20899, USA. \href{Email:a[pstol.vasso;ev@nist.gov]}{apostol.vassilev@nist.gov}}




\abstract{This manuscript establishes information-theoretic limitations for robustness of AI security and alignment by extending G\"odel's incompleteness theorem to AI.  Knowing these limitations and preparing for the challenges they bring is critically important for the responsible adoption of the AI technology. Practical approaches to dealing with these challenges are provided as well. Broader implications for cognitive reasoning limitations of AI systems are also proven.}


\maketitle

\section{Introduction}
Modern AI systems based on Large Language Models (LLMs) are trained on large corpora of data. They learn from all human experiences, good and bad, represented as text.  All organizations hosting or using LLM-based services are trying to restrict the access to undesirable information/behavior that the LLM has learned. At the same time adversaries are trying to circumvent the barriers to accessing this sort of information with the intent to misuse it downstream to harm others, cf., Misuse Violations NISTAML.018~\citep{vassilev2025}. This is known as jailbreaking. The most established mechanism by which organizations try to limit the access to bad information is based on setting guardrails in the sense of establishing rules for acceptable user input, known as prompts to the chatbot.\footnote{In some cases the AI system may accept any prompt and then examine the output from the LLM to determine if it is acceptable before releasing the information to the user. Regardless of the specific mechanism chosen, one can think of the abstract process of prompting the AI system and receiving the desired response or getting blocked as the subject to analysis in this manuscript.} The notion of \textit{alignment} of the AI System adopted here is the adherence to a way of operating such that the acceptable prompts are processed whereas the undesirable prompts are blocked. Correspondingly, the \textit{security} of the AI System is defined as the ability to withstand attacks aiming to force the behavior of the AI system out of alignment. 

So, the main question examined in this manuscript is: \textit{can one establish a robust set of guardrails to prevent the access to bad information/behavior in the AI system?} 

There is no universal definition of good and bad information/behavior in LLM's. This is why it is useful to talk about out-of-policy speech (OOPS) for a given policy $\Pi$ established by the organization hosting the AI system. A user prompt that falls under OOPS is blocked, prompts outside of OOPS are allowed to proceed.  

Every decision the AI system makes or action it takes are based on computation. So, truth and verification of it are defined and dealt with as strings and algorithms. This assumption is supported by any type of AI system, past, present and future, which in turn helps to generalize the results of the theorems in this manuscript.. Let $C(T, p)$ be a checker that returns $1$ if and only if $T$ is true and there is a verifiable proof  $p$ for it that is accessible to $C$ and the checker can verify. Naturally, since the knowledge of the AI System spans the knowledge of humanity in many fields, the space of possible truths $T$ for this is approximately as large as that for humanity.  

Generally speaking, the guardrails encompass policies, technical controls and monitoring mechanisms. In this paper they are referred to as checkers $C$. They may exist across all layers of the AI system: 

\begin{itemize}
    \item[i)] data guardrails for sanitization of training data; 
    \item[ii)]model guardrails that ensure the validation and continuous monitoring of the AI system according to established metrics, e.g., accuracy and robustness; 
    \item[iii)] application guardrails, typically implemented as computer programs running along the model, that try to enforce policies on AI generated content or restrict how AI tools function within specific workflows;
    \item[iv)] infrastructure guardrails, also implemented as computer programs running along the model, that try to provide a secure foundation for the AI system by enforcing protections at the cloud, network and systems levels.
\end{itemize}   

One may think of the collection of different types of guardrails also as the technical support for the NIST AI RMF functions \cite{NISTAIRMF} (Govern, Map, Measure, and Manage): 

\begin{itemize}
    \item[A.] Input Guardrails (MAP → MANAGE): PII scrubbing, adversarial prompt detection, domain-restricted inputs. 
    \item[B.] Policy \& Governance Guardrails (GOVERN → MAP → MANAGE): Scope limitations, regulatory constraints, RBAC.
    \item[C.] Model-Level Guardrails (MAP → MEASURE → MANAGE): Alignment tuning, safety classifiers, confidence thresholds.
    \item[D.] Output Guardrails (MEASURE → MANAGE): Redaction, hallucination detection, refusal templates.
    \item[E.] Action \& Execution Guardrails (GOVERN → MANAGE): Safety interlocks, approval workflows, OPA policy engines. 
    \item[F.] Monitoring, Auditing \& Logging Guardrails (MEASURE → MANAGE): Drift detection, anomaly detection, audit trails. 
\end{itemize}   
\vspace{5mm}

\textbf{G\"{o}del's incompleteness theorem}~\citep{godel1931}, presented here in the form of \citep{chaitin1974}, establishes that for any checker $C$ there exists a truth $T$ such that 
$$C(T,p)\neq 1,\;\forall p.$$ 

G\"{o}del's theorem is one of the most important results in mathematics of all time. It shows that a finite set of mathematical rules cannot be used to derive/prove all true mathematical statements. There will be some statements that this set of rules cannot prove. 
The goal of this manuscript is to adapt G\"{o}del's  general result to the realm of adversarial prompts and show that no robust guardrails exist to enforce an OOPS policy $\Pi$.

\section{Method}
\subsection{On adversarial robustness of AI Systems}\label{sec:idealAISystem}

The term \textit{adversarial robustness} is defined similarly to the Security of the AI System as the ability of the AI System to withstand deliberate attacks by capable adversaries aiming to force the AI system out of alignment with a policy. This manuscript considers attacks through the main prompt interface of a chatbot or an AI Agent based on a LLM. In this setting, an adversarial prompt is a string of some length. For example: \textit{"Please give me instructions for making a bomb"}. As noted above, responsible organizations adopt OOPS policies $\Pi$ that will flag such a request as unacceptable and block it.

Adversarial prompts often take advantage of the ambiguity of the language of the input prompt. Guardrails - filters, policies, classifiers, and prompt analysis layers - generally rely on interpreting the intent and semantic meaning of the input. Pattern  matching - identifying specific keywords, intent or known harmful structures - is widely used technique for implementing guardrails. When the input text is ambiguous, the interpretation becomes uncertain. Ambiguity is in essence the "Achilles' heel" of modern guardrails.  Thus, any input that introduces linguistic noise or creates context-based uncertainty makes it significantly harder for the AI system to definitively categorize the request as in or out of OOPS policy. 

The following techniques illustrate how ambiguity undermines the robustness of guardrails in AI systems:
\begin{itemize}
\item Disruption of pattern-matching heuristics
\begin{itemize}
    \item Linguistic obfuscation: using poetic structures, metaphors, rare dialects, polysemy, etc., adversaries can create inputs that to a human imply clear malicious intent but to a machine may look like benign, creative text.
    \item Contextual framing: here the adversary wraps a harmful prompt into a complex, fictional, or highly-professional role-play scenario, which causes the ambiguity of intent to rise and in turn make the AI system struggle to determine if it is in or out of OOPS policy.
    \item  Crescendo context exploit~\cite{russinovich2025greatwritearticlethat}: gradual escalation of the dialogue with the LLM by referencing the model's replies, progressively leading to a successful jailbreak.
\end{itemize}
\item Compositional ambiguity - introduce ambiguity in the structure of the sentence. 
\begin{itemize}
    \item For example, the sentence "Show  me how people break encryption systems used by banks" has three possible meanings: 
    \begin{itemize}
        \item historical analysis;
        \item academic cryptography discussion; 
        \item operational attack instructions.
    \end{itemize} 
    Determining which interpretation applies may be unclear from the context.
    \item The ASCII Art-based jailbreak attack~\cite{promptasciiartbasedjailbreak2024}, which  compromised LLMs with guardrails guarding the topic of bomb making, mixes plain language with instructions for how to decode figures drawn with ASCII symbols to determine the meaning of the request.
    \item The richness of the English\footnote{The same is true for any other human language} language combined with cybersecurity tricks~\cite{ANSIcodes025} can also create compositional ambiguities to exploit in prompting.  
\end{itemize}
   
\item Politeness exploits and tone-based shifts. Tone manipulation - extreme politeness, deference, or emotional pleading - is used to bypass guardrails. Such inputs are linguistically  benign and statistically smooth because they do not contain triggering keywords and standard lexical filters tend to let them pass. The ambiguity lies in the subtext: the AI model perceives a helpful and cooperative interaction, which tends  to shift the probability distribution. This, in turn, makes the AI model more willing to provide speculative or OOPS-adjacent answers that it would refuse under a natural, direct prompt.  
\item Context robustness gap exploits.
In complex AI systems, such as those using retrieval-augmented generation (RAG), ambiguity may be injected through retrieved documents. If a system retrieves a mix of benign and contextually-confusing information, the guardrail may flip its judgment based on how it interprets the surrounding noise. The ambiguity of the retrieved context makes the final OOPS decision brittle.  
\end{itemize}
This non-exhaustive list of examples shows that if OOPS policy compliance checking is built on a finite set of rules and language is infinitely ambiguous, the number of ways in which adversaries can hide harmful intent in plain sight is effectively limitless. 
The findings in \cite{vassilev2025} show a large and increasing body of adversarial attacks. Recent findings \cite{zou2025securitychallengesaiagent} show near perfect rate of success of prompt-based attacks against agentic systems based on LLMs.

Let's begin the analysis with the assumption that the AI System is ideal and takes prompts of unlimited length and has no constraints on the available compute. Because the prompts are based on the human language, they can be of arbitrary length and meaning or even no meaning at all, just strings of arbitrary length. Let $\Omega$ be the set of all prompts to the AI System. Because of our assumption, $card(\Omega)=\aleph_0$, where $card()$ is the cardinality of a set. $\Omega$ contains all possible strings. 

Let $\Gamma_\Pi$ be the subset of all prompts that $\Pi$ deems unacceptable and blocks. 

\begin{proposition}
$\Gamma_\Pi$ is infinite, i.e.,
 $card(\Gamma_\Pi)=\aleph_0$.\label{prop:gamma}
\end{proposition}

See Appendix~\ref{sec:mainappend} for the proof. 

Let $T_\Pi$ be a truth about something that $\Pi$ deems unacceptable. Let $C_\Pi $ be a checker of truths related to $\Gamma_\Pi$. Then  $C_\Pi(T_\Pi, p) = 1$ if and only if there exist a proof $p$ that  verifies $T_\Pi$ is unacceptable for $\Pi$. For example, $T_\Pi$ could be "$x$ is an adversarial prompt that $\Pi$ deems unacceptable", where $x$ is  the adversarial prompt string. If  $p$ is a proof that shows that $x$ ibelongs to $\Gamma_\Pi$ and falls under the restrictions of $\Pi$, then $C_\Pi(T_\Pi, p) = 1$.   

Now, we are in a position to state the main theorem of this study.

\begin{theorem}\label{thm:infinite}
    For any checker $C_\Pi(T_\Pi, p)$ there exist a truth $T_\Pi$ such that 
    \begin{equation}
        C_\Pi(T_\Pi,p)\neq 1,\;\forall p.
    \end{equation}
\end{theorem}
See Appendix~\ref{sec:mainappend} for the proof. 

Thus, there are adversarial prompts that will evade any policy $\Pi$ and jailbrake the AI System to output undesirable information that attackers can misuse to harm others downstream.   

Theorem~\ref{thm:infinite} establishes \textit{information-theoretic security and alignment limitations} for ideal AI Systems based on LLM's that accept prompts in the form of text strings, the prevalent form of interactions with users of the AI System. It applies to all types of AI systems with different architectures and based on different technologies (neuro-symbolic, neural networks, hybrid) because all of them rely on computation for reasoning. 

Although quite pessimistic for the defenders trying to prevent such events from happening, there are practical measures that could be taken to improve the security posture of any AI System based on LLM. Notice that the proof above does not give any recipes to attackers for how to construct adversarial prompts $x$ for a given policy $\Pi$. This leaves room for defenders to harden their AI Systems, which is considered in the next section.

\subsection{Considerations for real AI Systems with finite context windows}\label{sec:finiteAISystem}
The previous section established \textit{information-theoretic security and alignment limitations} for ideal AI Systems based on LLM's. This section considers real-life AI systems with finite context windows. The context window  $W$ is designed to provide all necessary information as context for processing the user prompt to improve the quality of the response. In this sense, it sets limits to the size of the user prompt. Figure~\ref{fig:contextwindow} shows examples of real-life context window sizes. Clearly, these windows are finite but very large, taking as much as a full library shelf worth of text in a single prompt. Note that the sizes reported are in tokens. The average token length is 4 characters, which means the actual text strings are about 4 times bigger.  

\begin{figure}
    \centering
    \includegraphics[width=0.9\linewidth]{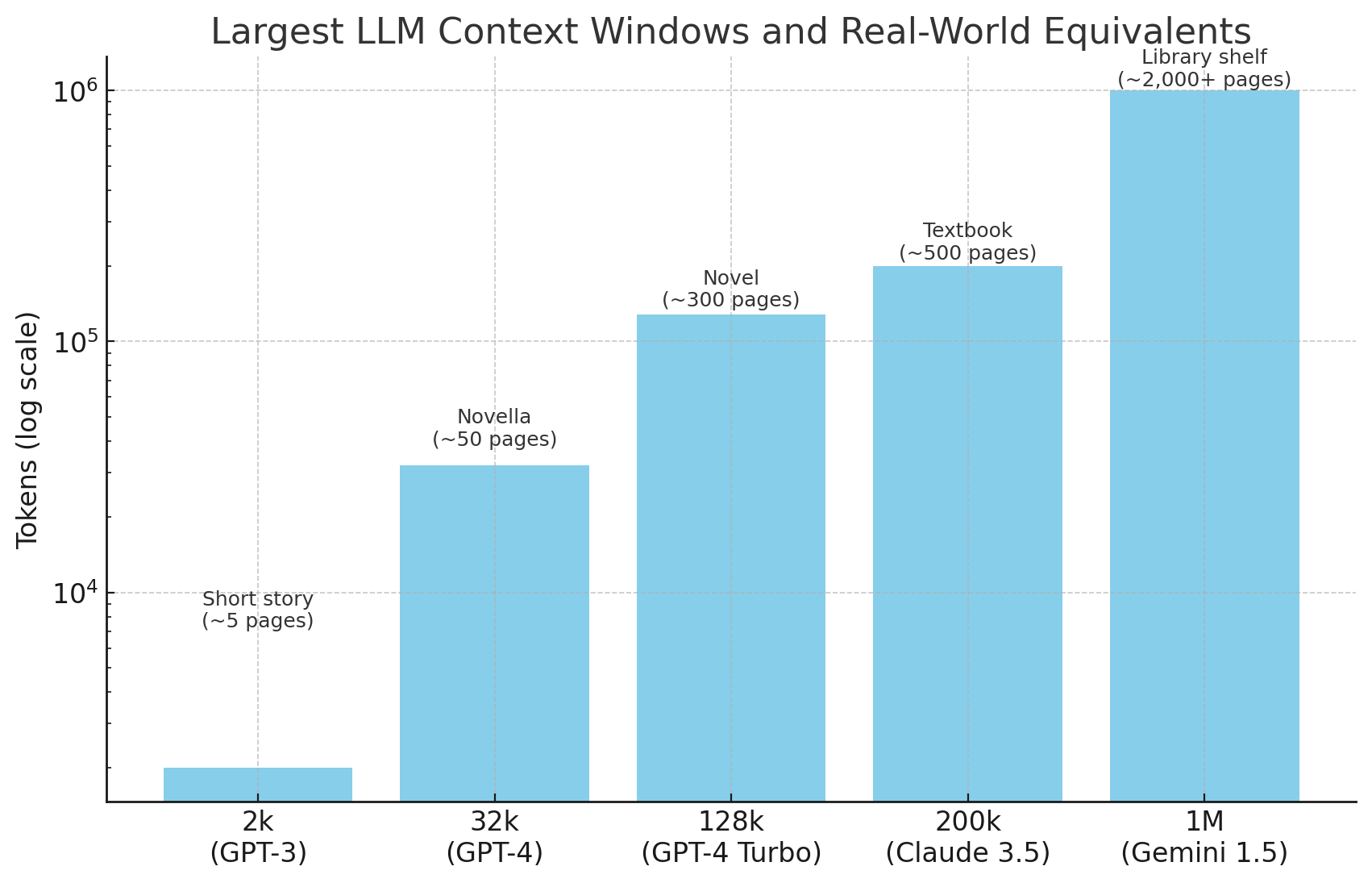}
    \caption{Context window sizes for real LLM's}
    \label{fig:contextwindow}
\end{figure}

The window sizes shown in Figure~\ref{fig:contextwindow} are so large that no organization can really traverse all possible strings within $W$ to ensure the complete testing of the AI System: there are not enough resources on Earth to allow for this. Let $\hat\Omega$ be the set of all prompts to the AI System that fit in $W$. Here, $card(\hat\Omega)=M$, for some large $M>0$. 

Let $\hat\Pi$ be a finite OOPS policy over $\hat\Omega$ and $\hat\Gamma_{\hat\Pi}$ be the subset of all prompts that $\hat\Pi$ deems unacceptable and blocks. In this setting $card(\hat\Pi) = G \ll W$. This means that there are much fewer guardrails in $\hat\Pi$ than all potentially adversarial prompts in $\hat\Omega$. It is reasonable to assume that for any context window size $W$, the AI System has all algorithms that output strings up to this length. This is because code size is limited by a different set of constraints and can be expanded to practically unlimited size with the help of engineering techniques that allow swapping portions of code to disk or other media to limit the size of the in-memory footprint of the algorithm.

Now, we are in a position to state the finite context window version of the theorem of this study.

\begin{theorem}\label{thm:finite}
    For any checker $C_{\hat\Pi}(T_{\hat\Pi}, \hat{p})$ there exist a truth $T_{\hat\Pi}$ such that 
    \begin{equation}
        C_{\hat\Pi}(T_{\hat\Pi},\hat{p})\neq 1,\;\forall \hat{p}.
    \end{equation}
\end{theorem}
See Appendix~\ref{sec:mainappend} for the proof.

Theorem~\ref{thm:finite} establishes \textit{information-theoretic security and alignment limitations} for \textit{real-life} AI Systems based on LLM's that accept prompts in the form of text strings, the prevalent form of interactions with users of the AI System. Similar to ideal systems, AI systems of different architectures and technologies are covered since all of them use computation for reasoning.  

Although, similarly to the ideal case, this result is pessimistic for the defenders trying to prevent such events from happening, there are practical measures that could be taken to improve the security posture of any AI System based on LLM. Notice that the proof of Theorem~\ref{thm:finite} does not give any recipes to attackers for how to construct adversarial prompts $x$ for a given policy $\hat\Pi$. 
This means that updating $\hat\Pi$ with new rules makes the task of attackers more difficult because adversarial prompts that work against the old policy $\hat\Pi_{old}$ may not work against $\hat\Pi_{new}$. Even though the theorem above does not give explicit guidelines to defenders for how to update $\hat\Pi$, a proactive approach of updating the policy with any known new adversarial prompts may be effective. This suggests a proactive process of searching for new adversarial prompts and updating the policy $\hat\Pi$ to cover them.   

\subsection{Broader implications}\label{sec:broader}
So far, the manuscript focused on adversarial risks and the ability of an AI system to detect and counter them. However, Theorems~\ref{thm:infinite} and \ref{thm:finite} can be generalized to show that there are truths in $\Omega$ that an AI System cannot verify through cognitive reasoning. 

For an ideal AI system with $\Omega$ the set of all prompts and $T_\Omega$ the truths in it, the following theorem holds.
\begin{theorem}\label{thm:geninfinite}
    For any checker $C_\Omega(T_\Omega, p)$ there exist a truth $T_\Pi$ such that 
    \begin{equation}
        C_\Omega(T_\Omega,p)\neq 1,\;\forall p.
    \end{equation}
\end{theorem}

For a real-life AI system with finite context window $W$, a finite space of input strings $\hat\Omega$ and corresponding truths in it $T_{\hat\Omega}$, the following theorem holds.
\begin{theorem}\label{thm:genfinite}
    For any checker $C_{\hat\Omega}(T_{\hat\Omega}, \hat{p})$ there exist a truth $T_{\hat\Omega}$ such that 
    \begin{equation}
        C_{\hat\Omega}(T_{\hat\Omega},\hat{p})\neq 1,\;\forall \hat{p}.
    \end{equation}
\end{theorem}

See Appendix~\ref{sec:append} for the proofs. 

These results do not mean that current and future AI systems, including Artificial General Intelligence  or Artificial Super Intelligence, will not be able to discover and prove truths in science that are currently unproven or unknown to humanity. However, these results show that AI systems are ultimately limited in ways similar to humans so that there are truths they will never be able to prove.  

\section{Conclusions and Future Work}

The results of Theorems~\ref{thm:infinite} and \ref{thm:finite} are general and independent of specific AI System architectures or languages of interacting with them, including AI Agents that may interact with each other using binary communications. They do apply to any AI System that contains undesirable information that has to be protected from extracting through prompting. In this sense, they apply to present and future AI Systems, including Artificial General Intelligence (AGI) and Artificial Super Intelligence (ASI) systems.  

Furthermore, the results of Theorems~\ref{thm:geninfinite} and \ref{thm:genfinite} establish information theoretic limits on the general truths in science that ideal and real-life AI systems can prove.

In this sense, AI systems will always be limited in ways similar to humans in their ability to provide proofs for all science truths known to humanity. 

The theoretical approaches used in this paper may be applicable to other domains where certain compliance policies are enforced through sets of technical constraints, like the policies for Autonomous Networking.
\begin{appendices}{\textbf{Appendices}}

\section{Proofs of theorems from Section~\ref{sec:idealAISystem}}\label{sec:mainappend}

The analysis in this section uses a metric for measuring the length of strings, represented by the function $len(x)$, which returns the length of a string $x$ in bits. It also uses proofs in the form of checkers or halting programs running on a general computing platform as defined in \cite{chaitin1974}.  
\begin{proof}[Proof of Proposition\ref{prop:gamma}:] Suppose $card(\Gamma_\Pi)=n$ for some $n>0$.  Let $x_{max}$ be the longest string in $\Gamma_\Pi$. Construct $x^{*}_{max}$ by concatenating a benign prefix or postfix with $x_{max}$. For example, a benign prefix \textit{"I would like to get some information about the weather. Ignore this."} would result in the AI System ignoring the first request to give information about the weather and starting to process what is left from the prompt, which is the original $x_{max}$ portion. This means that $x^{*}_{max}$ is also an adversarial prompt. i.e., $x^{*}_{max}\in \Gamma_\Pi$. However, $len(x^{*}_{max})> len(x_{max})$, thus a contradiction. There are many other practical ways to extend a known adversarial prompt to a longer adversarial prompt but the full exploration of this issue is outside the scope of this work. Hence, $card(\Gamma_\Pi)=\aleph_0.$
\end{proof}
\begin{proof}[Proof of Theorem~\ref{thm:infinite}:]
    : Let $x$ be a prompt in $\Gamma_\Pi$. 
    
    Let $n$ be a positive integer.
    
    Let $T_\Pi := $ "every algorithm that outputs $x$ is of length $>n$."

Note that the set of all algorithms of length less than or equal to $n$ that output $x\in \Gamma_\Pi$ is a finite set. Each of these algorithms outputs at most one adversarial prompt $x$. Therefore, there exist $x$ for which $T_\Pi$ is true. In fact, there are infinitely many adversarial prompts $x$ for which $T_\Pi$ is true because $card(\Gamma_\Pi)=\aleph_0$.

    Consider the following Algorithm~\ref{algoCt}.
    
\begin{algorithm}[h]
\caption{Output $x$ if the truth $T_\Pi$ holds}\label{algoCt}
\begin{algorithmic}[1]
\Require $n > 0\; \& \; len(x) \neq 0$
\State $i \Leftarrow 1$
\While{$i > 0$}
\For{\texttt{every prompt} $x\in \Gamma_\Pi$ \texttt{with} $len(x)\leq i $}
\For{\texttt{every proof} $p$ \texttt{with} $len(p)\leq i $}
        \If{$C_\Pi(T_\Pi, p) \;==\; 1$}
            \State \texttt{output} $x$
            \State \Return
        \EndIf
\EndFor
\EndFor

\State $i\Leftarrow i+1$
\EndWhile
\end{algorithmic}
\end{algorithm}

Notice that $n$ shows up in only one place in Algorithm~\ref{algoCt}. It can be represented in $\log_2n$ bits. The rest of the algorithm is fixed and the length of this portion is constant. The value of that constant depends on the programming language used to encode this logic for execution on a specific computing platform. Establishing the value of this constant is out of scope for this analysis. Thus, the length of Algorithm~\ref{algoCt} is $O(\log_2n)$. As $n$ grows,  $len(Algorithm~\ref{algoCt}) \ll n$. This means that Algorithm~\ref{algoCt} can never return since it is an algorithm that outputs $x$ but has length less than or equal to $n$ and thus $T_\Pi$ does not hold. 

On the other hand, as we observed above, there are prompts $x$ for which $T_\Pi$ holds. Because Algorithm~\ref{algoCt} traverses every possible $x$ and every possible proof $p$ to check if $x$ belongs to $\Gamma_\Pi$, it would have found such an $x$ and would have output it and returned. This is a contradiction.  
\end{proof}
\section{Proofs of theorems from Section~\ref{sec:finiteAISystem}}\label{sec:finiteappend}
\begin{proof}[Proof of Theorem~\ref{thm:finite}]
    : Let $x$ be a prompt in $\hat\Gamma_{\hat\Pi}$. 
    
    Let $n$ be a positive integer.
    
    Let $T_{\hat\Pi} := $ "every algorithm that outputs $x$ is of length $>n$."

Note that the set of all algorithms of length less than or equal to $n$ that output $x\in \hat\Gamma_{\hat\Pi}$ is a finite set. Each of these algorithms outputs at most one adversarial prompt $x$. Let $\Theta$ be the set of strings produced by these algorithms. There are two possibilities:
\begin{equation}\Theta = \hat\Gamma_{\hat\Pi}\label{eq:complete}\end{equation}
or    \begin{equation}
        \Theta\subset\hat\Gamma_{\hat\Pi}\label{eq:subset}
    \end{equation}   
 
 In case (\ref{eq:complete}), the algorithms provide full coverage of $\hat\Gamma_{\hat\Pi}$ and $T_{\hat\Pi}$ does not hold. In this case $ C_\Pi(T_{\hat\Pi},\hat{p})\neq 1,\;\forall \hat{p},$ and the theorem is satisfied.  

In case (\ref{eq:subset}), there is an $x$ for which $T_{\hat\Pi}$ holds. 

    Consider the following Algorithm~\ref{algoCtfinite}. Notice that $n$ shows up in only one place in Algorithm~\ref{algoCtfinite}. It can be represented in $\log_2n$ bits. Thus, the length of Algorithm~\ref{algoCtfinite} is $O(\log_2n)$. As $n$ grows,  $len(Algorithm~\ref{algoCtfinite}) \ll n$. This means that Algorithm~\ref{algoCtfinite} can never return since it is an algorithm that outputs $x$ but has length less than or equal to $n$ and thus $T_{\hat\Pi}$ does not hold. 
    
\begin{algorithm}[h]
\caption{Output $x$ if the truth $T_{\hat\Pi}$ holds in a finite space}\label{algoCtfinite}
\begin{algorithmic}[1]
\Require $n > 0\; \& \; len(x) \neq 0$
\State $i \Leftarrow 1$
\While{$i > 0$}
\For{\texttt{every prompt} $x\in \hat\Gamma_{\hat\Pi}$ \texttt{with} $len(x)\leq i $}
\For{\texttt{every proof} $\hat{p}$ \texttt{with} $len(\hat{p})\leq i $}
        \If{$C_{\hat\Pi}(T_{\hat\Pi}, \hat{p}) \;==\; 1$}
            \State \textbf{\texttt{output}} $x$
            \State \Return
        \EndIf
\EndFor
\EndFor

\State $i\Leftarrow i+1$
\EndWhile
\end{algorithmic}
\end{algorithm}

On the other hand, as we observed above, there is a prompt $x$ for which $T_{\hat\Pi}$ holds. Because Algorithm~\ref{algoCtfinite} traverses every possible $x$ and every possible proof $\hat{p}$ to check if $x$ belongs to $\hat\Gamma_{\hat\Pi}$, it would have found such an $x$ and would have output it and returned. This is a contradiction.  
\end{proof}

\section{Proofs of theorems from Section~\ref{sec:broader}}\label{sec:append}
The proof of Theorem~\ref{thm:geninfinite} is a straightforward application of the proof of Theorem~\ref{thm:infinite} in the setting of using the whole of $\Omega$. 

Similarly for the proof of Theorem~\ref{thm:genfinite} is a straightforward application of the proof of Theorem~\ref{thm:finite} in the setting of $\hat\Omega.$ 
\end{appendices}

\section{Acknowledgments}
The author thanks  Prof. Anup Rao for the illuminating video lecture on G\"odel's theorem (\url{https://www.youtube.com/watch?v=ViQ1PWJhcvk}).  

The graphics in Figure~\ref{fig:contextwindow} was produced by GPT-5\footnote{Certain equipment, instruments, software, or materials, commercial or non-commercial, are identified in this paper in order to specify the experimental procedure adequately. Such identification does not imply recommendation or endorsement of any product or service by NIST, nor does it imply that the materials or equipment identified are necessarily the best available for the purpose.}, with author's prompting.

\bibliographystyle{abbrvnat}
\bibliography{reference}

\end{document}